\documentclass{article}
\usepackage{arxiv}
\usepackage{amsmath}
\usepackage{amsthm}
\usepackage{graphicx}
\usepackage{algorithm}
\usepackage{algorithmic}
\usepackage{color}
\usepackage{cite}

\newcommand{\temax}{t_{\mathrm{max}}}
\newcommand{\tepoch}{t_{\mathrm{epoch}}}

\newcommand{\Bhat}{\hat{B}}
\newcommand{\onelambda}{$(1+\lambda)$-EA\xspace}
\newcommand{\oneone}{$(1+1)$-EA\xspace}
\usepackage{xspace,dsfont}
\newcommand{\ignore}[1]{}

\usepackage{amsmath,amssymb,amsfonts}
\def\P#1{\textsf{P}\{{#1}\}}
\def\bigP#1{\textsf{P}\left\{{#1}\right\}}
\def\Real{{\mathbb R}}
\usepackage{multirow}
\let\texdisplaystyle\displaystyle
\def\displaytotextstyle{\textstyle\let\displaystyle\texdisplaystyle}

\newenvironment{talign}
 {\let\displaystyle\displaytotextstyle\align}
 {\endalign}
\newenvironment{talign*}
 {\let\displaystyle\displaytotextstyle\csname align*\endcsname}
 {\endalign}
\newtheorem{definition}{Definition}
\newtheorem{lemma}{Lemma}
\newtheorem{theorem}{Theorem}
\newtheorem{corollary}{Corollary}

\newcommand{\gsemo}{GSEMO\xspace}
\newcommand{\swgsemo}{SW-GSEMO\xspace}

\title{Archive-based Single-Objective Evolutionary Algorithms for Submodular Optimization}

    \author{Frank Neumann\\
Optimisation and Logistics,\\
School of Computer and Mathematical Sciences,\\
The University of Adelaide,\\
Australia
\And
 G{\"u}nter Rudolph\\
Computational Intelligence,\\
Department of Computer Science,\\ TU Dortmund University,\\ Germany
}

\begin{document}
\maketitle              
\begin{abstract}
Constrained submodular optimization problems play a key role in the area of combinatorial optimization as they capture many NP-hard optimization problems. So far, Pareto optimization approaches using multi-objective formulations have been shown to be successful to tackle these problems while single-objective formulations lead to difficulties for algorithms such as the \oneone due to the presence of local optima. We introduce for the first time single-objective algorithms that are provably successful for different classes of constrained submodular maximization problems. Our algorithms are variants of the \onelambda and \oneone and increase the feasible region of the search space incrementally in order to deal with the considered submodular problems.

\end{abstract}

\section{Introduction}
Many combinatorial optimization problems that face diminishing returns can be stated in terms of a submodular function under given set of constraints~\cite{DBLP:books/cu/p/0001G14}.
The maximization of a non-monotone submodular function even without constraints includes the classical maximum cut problem in graphs and is therefore an NP-hard combinatorial optimization problem that cannot be solved in polynomial time unless $P = NP$ but different types of approximation algorithms are available \cite{DBLP:journals/siamcomp/FeigeMV11}. Monotone submodular functions play a special role in the area of optimization as they capture import coverage and influence maximization problems in networks.
The maximization of monotone submodular functions is NP-hard even for the case of simple constraint that limits the number of elements that can be chosen, but greedy algorithms have shown to obtain best possible approximation guarantees for different types of constraints~\cite{DBLP:journals/mp/NemhauserWF78,DBLP:books/cu/p/0001G14}.
At best, one can hope to develop a method that can provide an $\alpha$-approximation in polynomial time, i.e., a solution with a value of at least $\alpha\,f(x^*)$ where $\alpha\in (0,1)$ and $x^*$ is an optimal solution of the submodular function $f(\cdot)$. Such an algorithm was proposed in \cite{DBLP:journals/mp/NemhauserWF78} where it was proved that a greedy method can find an $(1-1/e)$-approximation of the maximum of a submodular function in polynomial time.

Although the $(1+1)$-EA shares many characteristics with a greedy algorithm, it was proven in \cite[Thm.~1]{DBLP:journals/ec/FriedrichN15}, that it can get trapped in local optima even for monotone submodular problems with a uniform constraint requiring exponential time to achieve an
approximation better than $1/2 + \varepsilon$ for any given $\varepsilon > 0$.

Due to this disappointing result, the focus shifted to other types of evolutionary algorithms. Since multiobjective EAs have proven successful in the treatment of combinatorial problems in the past \cite{DBLP:journals/nc/NeumannW06,DBLP:journals/ec/FriedrichHHNW10,DBLP:conf/emo/KnowlesWC01}, the variant GSEMO \cite{GielCEC2003} has been applied to the maximiziation of a submodular function with cardinality constraint. Guided by the proof in \cite{DBLP:journals/mp/NemhauserWF78} it was proven in \cite{DBLP:journals/ec/FriedrichN15} that the GSEMO can find a $(1-1/e)$-approximation in polynomial time with small failure probability. In the sequel there have been several publications in this direction considering different variants of GSEMO together with appropriate muli-objective formulations treating the considered constraint as an additional objective~\cite{DBLP:conf/ijcai/QianSYT17,DBLP:conf/ijcai/Crawford21,DBLP:conf/nips/QianYZ15,DBLP:journals/ai/QianYTYZ19,DBLP:journals/ai/RoostapourNNF22,DBLP:conf/ppsn/NeumannN20}.

Recently, the \emph{sliding window} GSEMO (SW-GSEMO) has been introduced in~\cite{NeumannWittECAI23} which outscores the performance of the original GSEMO significantly of large problem instances. The improvement here comes from a sliding window selection method that selects parent individuals dependent on the anticipated progress in time during the optimization run.
Motivated by the insights gained through the development of SW-GSEMO, we show that singleobjective EAs with only few small algorithmic changes to the standard versions are able 
to achieve the same theoretical and competitive practical performance as their multiobjective counterpart, i.e., it is not necessary to apply \emph{multiobjective} EAs. This result potentially opens a new area in the development of evolutionary algorithms for submodular problems which has so far relied on the use of multi-objective problem formulations and algorithms.

The outline of the paper is as follows.
In Section \ref{sec:prelim} we introduce terminology and basic results regarding submodular functions. Section \ref{sec:noArc} provides theoretical results for single-objective EAs without archive, whereas Section~\ref{sec:withArc} presents the proof that an (1+1)-EA can successfully solve the submodular problem if it is equipped with a specific kind of archive. The theoretical findings are supported by experimental results on graph cover problems in Section~\ref{sec:Exp}. Finally, Section~\ref{sec:concl} reports on our conclusions.

\section{Preliminaries}\label{sec:prelim}

\begin{definition}
Let $U$ be a finite ground set and $f:2^U\to\Real^+_0$. If for all $A,B \subseteq U$ holds\\
a) $f(A\cup B) + f(A\cap B) \le f(A) + f(B)$
then $f$ is termed \emph{submodular};\\
b) $f(A)\le f(B)$ then $f$ is called \emph{monotone}.
\end{definition}

Many functions arising in combinatorial optimization are submodular. For example, let $A_1,\ldots,A_n$ be subsets of a finite universe $U$. Then the \emph{coverage function} 
$f(S) = |\cup_{i\in S} A_i|$ with $S\subseteq\{1,\ldots,n\}$ is submodular. Submodular functions are also called functions of diminishing returns, as demonstrated in Figure~\ref{fig:setCover}: The later we add the blue set to the increasing gray set, the smaller is the gain of area.

\begin{theorem}[see \cite{DBLP:journals/mp/NemhauserWF78}, Proposition 2.1]\label{T:equiv}
The following conditions are equivalent to the definition of submodular set functions:\\
a) for all $A\subseteq B\subseteq U$ and $x\notin B$
$$f(A\cup \{x\}) - f(A) \ge f(B\cup \{x\}) - f(B)$$
b) for all $A\subseteq B\subseteq U$
$$f(B) \le f(A) + \sum_{x\in B\setminus A} (f(A\cup\{x\})-f(A))$$
\end{theorem}

\begin{figure}[t]
\centering
\includegraphics[width=0.33\textwidth]{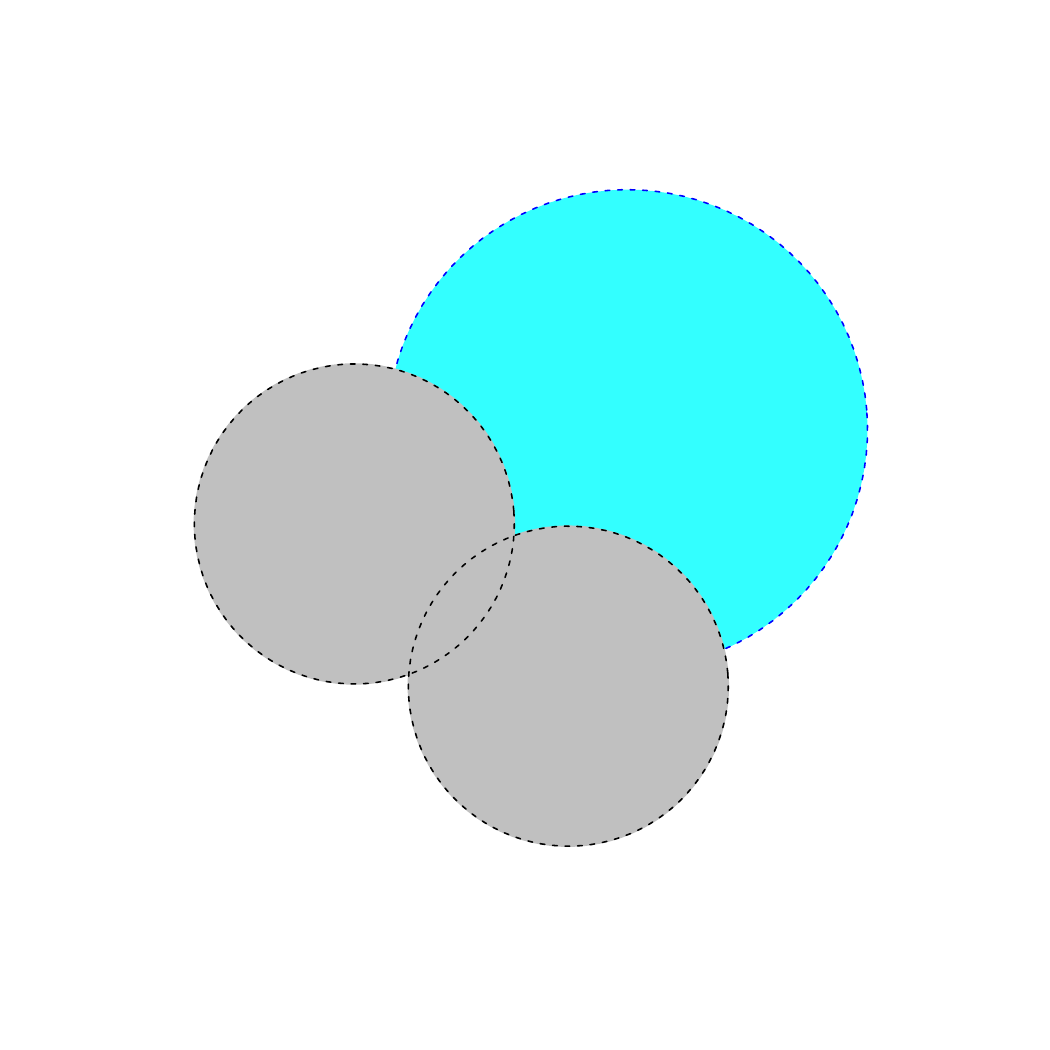}
\hfill\includegraphics[width=0.33\textwidth]{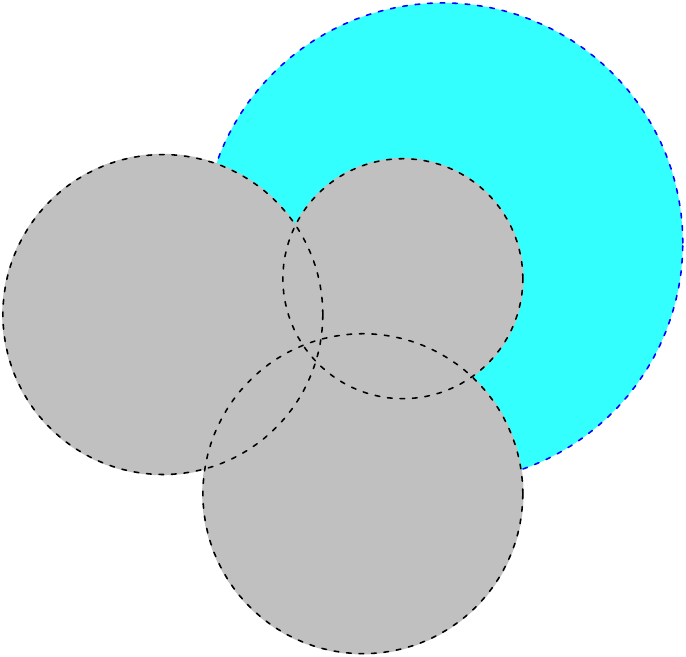}\hfill\includegraphics[width=0.33\textwidth]{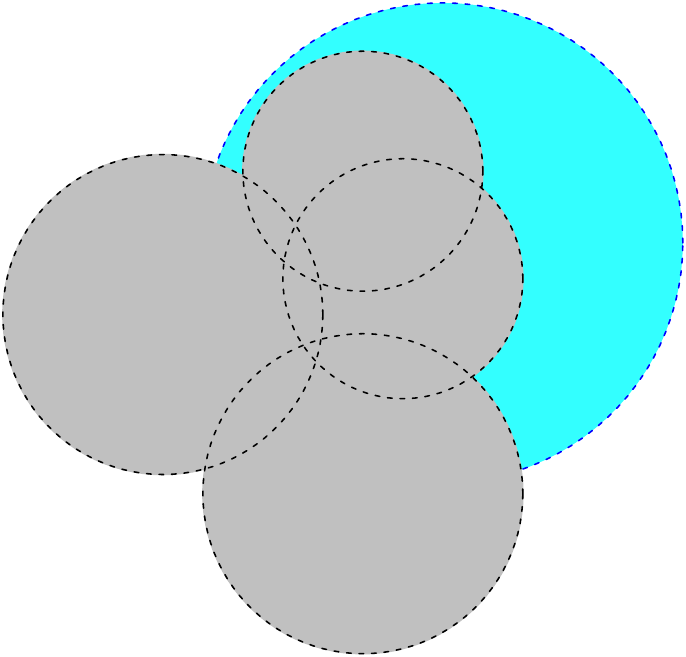}
\caption{Gray set $A$ in the left figure is a subset of the gray set $B$ in the middle figure which in turn is a subset of the gray set $C$ in the right figure. Adding the blue set leads to a lower gain of area the larger the gray set is.}
\label{fig:setCover}
\end{figure}

For later purpose we present an auxiliary result which can be established from well-known properties of monotone submodular functions~\cite{DBLP:journals/mp/NemhauserWF78,DBLP:books/cu/p/0001G14}. We present it here together with the proof as it is crucial for optimizing monotone submodular functions with a uniform constraint.
\begin{lemma}\label{L:delta}
Let $f$ be a monotone submodular set function, $X^*$ an optimal solution and $X$ some feasible solution. Then
\begin{equation}\label{eq:delta}
f(X^*) \le f(X) + r\,\delta
\end{equation}
where $\delta=\max\limits_{x\in X^*\setminus X}(f(X\cup\{x\})-f(X))$ and $r=|X^*|$.
\end{lemma}
\begin{proof}
Note that we have $f(X^*) \leq f(X \cup X^*)$ since $f$ is monotone.
Let 
$$\delta = \max_{x \in X^* \setminus X} (f( X \cup \{x\}) - f(X))$$
be the largest marginal gain among all elements in $X^* \setminus X$.
We have 
\begin{eqnarray}
f(X^*) & \leq & f(X \cup X^*) \nonumber\\
& \leq & f(X) + \sum_{x \in X^* \setminus X} (f( X \cup \{x\}) - f(X)) \label{eq:L-delta}\\
& \leq & f(X) + |X^*| \cdot \delta \nonumber\\
& = & f(X) + r \cdot \delta. \nonumber
\end{eqnarray}
where inequality (\ref{eq:L-delta}) follows from Theorem \ref{T:equiv}(b).
\end{proof}

\begin{definition}
Let $c: 2^U\to\Real_+$ with $|U| = n<\infty$ and budget $B > 0$. 
The constraint $c(X) \le B$ is termed a \emph{cardinality} or \emph{uniform constraint} if $c(X) = |X|$ for $X\in 2^U$ and $B \le n$. Otherwise it is called a \emph{general constraint}.
\end{definition}

\begin{definition}
The maximization of a monotone submodular function under a given constraint is termed the \emph{monotone submodular maximization problem} (MSMP).
\end{definition}

\begin{definition}
\label{def:subratio}
    The submodularity ratio $\alpha_f$  of a non-negative set function $f$ is defined as
    $\alpha_f = \min_{X \subseteq Y \subseteq U, v \not \in Y} \frac{f(X \cup v) - f(X)}{f(Y \cup v) - f(Y)}.$
\end{definition}
A function $f$ is submodular iff $\alpha_f=1$ holds. In Section~\ref{sec:withArc}, we will consider general monotone objective and monotone cost functions and investigate approximations dependent of $\alpha_f$. 

When considering evolutionary algorithms for the optimization of submodular function, we work with the search space $\{0,1\}^n$, i.e. search points are binary strings of length $n$. We identify each element $u_i \in U$ with a bit $x_i$, $1 \leq i \leq n$, and define the set $X \subseteq U$ as $X =\{u_i \in U \mid x_i=1\}$. To ease the presentation we use the search point $x$ and its set of chosen elements $X$ in an interchangeable way.


\section{\onelambda without Archive}\label{sec:noArc}
We first consider the case of the optimization of a monotone submodular functions with a uniform constraint, i.e. $|x|_1 =\sum_{i=1}^n x_i \leq B$ holds for any feasible solution $x \in \{0,1\}^n$, before we consider the case of general constraints.

\subsection{Algorithm}

The $(1+\lambda)$-EA always starts at the zero string $x_j=0^n$ where $j=0,1, \ldots B$ denotes the $j$th epoch on the way to reach the bound $B$. The current bound $\Bhat$ is set to zero initially. It is clear that $x_0$ is feasible.

In each epoch $j\ge 0$ the $(1+\lambda)$-EA samples $\lambda$ offspring independently by mutation. 
For our theoretical investigations, we consider standard-bit-mutation which flip each bit independently of the others with probability $1/n$. For our experimental investigations, we consider standard-bit-mutation-plus as done in \cite{NeumannWittECAI23} which repeats standard-bit-mutation until at least one bit has been flipped.
We are seeking the best solution for the incremented bound $\Bhat$. If an offspring $y$ is feasible and not worse than its parent $x_j$ then it is accepted as a candidate for selection. After all $\lambda$ offspring have been evaluated the best candidate becomes the new best individual of epoch $j+1$.


This process repeats until the current bound $\Bhat$ exceeds the maximum bound $B$. The $(1+\lambda)$-EA is given in Algorithm \ref{alg:1plusLambda-new} in case of uniform constraint. 

\begin{algorithm}[t]
\begin{algorithmic}[1]
\STATE Set $j:=0, x_j:=0^n$, $\Bhat:=0$
\WHILE{$\Bhat < B$}
  \STATE $\Bhat = \Bhat + 1$ \label{line:incB}
  \STATE $\hat{x}=x_j$
  \FOR{$k:=1$ to $\lambda$}
    \STATE $y:=$ mutation($x_j$)
    \IF{$c(y) \le \Bhat$ \AND $f(y) \geq f(\hat{x})$}
      \STATE $\hat{x} := y$
    \ENDIF
  \ENDFOR
  \STATE $x_{j+1}:=\hat{x}$
\STATE $j:=j+1$ 
\ENDWHILE
\RETURN $x_j$
\end{algorithmic}
\caption{(1+$\lambda$) EA, input: $f, c, B, \lambda$}
\label{alg:1plusLambda-new}
\end{algorithm}

\subsection{Uniform Constraint}
\label{sec:noarcuniform}
\begin{theorem}
\label{thm:uniform}
The $(1+\lambda)$-EA finds a $(1-\frac{1}{e})$-approximation of a monotone submodular maximization problem with uniform constraint in at most $t_{\max} = 2\,e\,r\,n\,\log(n)$ function evaluations with probability $1-o(1)$, where $\lambda=2\,e\,n\,\log(n)$, $r$ is equal to the maximum budget in the constraint and $n$ is the dimension of the problem. 
\end{theorem}
\begin{proof}
The proof is oriented at the proofs of Theorem 1 in
\cite{NeumannWittECAI23} and Theorem 2 in 
\cite{DBLP:journals/ec/FriedrichN15} with adaptation to the context of the $(1+\lambda)$-EA.

Let $x^*$ be the optimal solution and $f(x^*)$ denote the global maximum. Assume that at each epoch $j=0,1,\ldots,r$ the EA has found a solution $x_j$ with at most $j$ elements such that
\begin{equation}\label{eq:bound}
f(x_j) \ge \left[1-\left(1-\frac{1}{r}\right)^j\right] \cdot f(x^*).
\end{equation}
If the assumption is true, then $x_r$ has the desired approximation ratio as can be seen from
$$f(x_r)\ge \left[1-\left(1-\frac{1}{r}\right)^r\right] \cdot f(x^*)\ge \left(1-\frac{1}{e}\right)\cdot f(x^*).$$
Therefore we have to establish the validity of inequality (\ref{eq:bound}) for all $r=0,1\ldots, r$, which is done by induction.

We begin with the feasible solution $x_j = (0,\ldots,0)$ at epoch $j=0$. Evidently, inequality (\ref{eq:bound}) is fulfilled since $f(x_0) \ge 0$. Now assume that $x_j$ is the current solution at time $\tau_j = j\cdot\lambda$ for $j=0,1,\ldots, r-1$. Now we can make $\lambda$ trials to find the best feasible improvement by mutation. For the best feasible improvement only a single specific bit mutation from $0$ to $1$ is necessary. As a consequence, the probability to transition from $x_j$ to $x_{j+1}$ in a single trial is lower bounded via
\begin{equation}\label{eq:transition}
\P{x_j\to x_{j+1} \text{ in single trial}} = \frac{1}{n}\,\left(1-\frac{1}{n}\right)^{n-1} \ge \frac{1}{e\,n}.
\end{equation}
As a consequence, the probability that the transition to $x_{j+1}$ does not happen in $\lambda$ trials is upper bounded by
\begin{equation*}
\begin{split}
\P{x_j\not\to x_{j+1} \text{ in } \lambda \text{ trials}} & \le
\left(1-\frac{1}{e\,n}\right)^\lambda 
=  \left(1-\frac{1}{e\,n}\right)^{2\,e\,n\,\log(n)} \\
 & =  \left[\left(1-\frac{1}{e\,n}\right)^{e\,n}\right]^{2\,\log(n)}
\le e^{-2\,\log(n)} = \frac{1}{n^2}.
\end{split}
\end{equation*}
Owing to the above bound and Boole's inequality we finally obtain
$$\bigP{\bigcup_{j=1}^r\{x_j \text{ not generated at } \tau_j\}}
\le \sum_{j=1}^r \P{x_j \text{ not generated at } \tau_j}
\le \frac{r}{n^2} \le \frac{1}{n}$$
since $r\le n$. This bound on the failure probability proves the success probability $1-o(1)$ in the statement of the theorem.

It remains to prove the induction step. According to (\ref{eq:delta}) in Lemma \ref{L:delta} we have $f(x^*) \le f(x_i)+r\,\delta_{i+1} \Leftrightarrow \delta_{i+1} \ge \frac{1}{r}\,(f(x^*)-f(x_i))$. It follows that
\begin{talign}
f(x_{j+1}) & \ge f(x_j) + \frac{1}{r}\,(f(x^*)-f(x_j)) \nonumber\\
& = f(x_j)\left(1-\frac{1}{r}\right)+\frac{1}{r}\,f(x^*) \nonumber\\
& \ge \left(1-\left(1-\frac{1}{r}\right)^j\right)\cdot f(x^*)\cdot\left(1-\frac{1}{r}\right) + \frac{1}{r}\,f(x^*)\label{eq:InductionStep}\\
&= \left(1-\frac{1}{r}\right)f(x^*)-\left(1-\frac{1}{r}\right)^{j+1}\,f(x^*)+\frac{1}{r}\,f(x^*)\nonumber \\
&= f(x^*)\left(1-\left(1-\frac{1}{r}\right)^{j+1}\right) \nonumber
\end{talign}
where (\ref{eq:InductionStep}) results from inserting the induction hypothesis.
\end{proof}



\subsection{General Constraint}
The \onelambda does not work in the more general case where the constraint is given by a linear function or the general cost function considered in \cite{DBLP:conf/ijcai/QianSYT17,NeumannWittECAI23} as can be observed by the following example instance of the classical knapsack problem.
Consider the knapsack problem where each item $i$ has weight $w_i$ and profit $p_i$ and the sum of the weight of the chosen items in any feasible solution is at most $B$.
Assume for items $i$, $1 \leq i \leq n-1$, we have $w_i=1$ and $p_i=1$ and for item $n$ we have $w_n=n-1$ and $p_n=L$ where $L$ is a large value, e.g. $L=2^n$. We set $B=n-1$. The optimal solution consists of the item $n$ only and any solution not chosen item $n$ has profit at most $n-1$. Note that choosing item $n$ plus any other item leads to an infeasible solutions. Hence, any feasible solution that is not optimal has approximation ratio at most $n/L$ which is $n/2^n$ for $L=2^n$.
The \onelambda starts with the solution $0^n$ and increases the bound iteratively. Only once $\Bhat=n-1$ holds, item $n$ may be introduced. However, once $\Bhat=n-1$ holds, a large number of the first $n-1$ items has been introduced which prevents element $n$ from being inserted. Inserting element $n$ is then only possible if all other elements are removed in the same mutation step. This leads with high probability to an exponential runtime for obtaining the optimal solution in the case of $\lambda=2e n \log n$ as chosen in Theorem~\ref{thm:uniform}. Even for $\lambda=1$, the algorithm would have included a constant fraction of the first $n-1$ elements before $\Bhat=n-1$ holds which again implies an exponential optimization time with high probability for $\lambda=1$. The arguments can be generalized to show an exponential optimization time with high probability for the \onelambda as defined in Algorithm~\ref{alg:1plusLambda-new} for any $\lambda \geq 1$.

\section{\oneone with Archive}\label{sec:withArc}
We now consider the case of general (monotone) objective function $f$ and cost function $c$ as already investigated in \cite{DBLP:conf/ijcai/QianSYT17,NeumannWittECAI23} for variants of GSEMO using multi-objective formulations. Recall that the submodularity ratio $\alpha_f$ of a given function $f$ measures how close the function is of being submodular (see Definition~\ref{def:subratio}).

\subsection{Algorithm}
For the general setting, we consider a variant of the classical \oneone. The algorithm is shown in Algorithm~\ref{alg:1plus1-new}. It starts with the solution $x_0=0^n$ and a constraint bound $\Bhat=0$.
As for \onelambda, we use standard-bit-mutation for our theoretical investigations and standard-bit-mutation-plus in the experiments.
For $\tepoch$ iterations, the single solution is improved under the current bound and solutions that are currently infeasible but still meet the bound $B$ of the given problem are added to an archive $A$. After the current epoch is finished the bound $\Bhat$ is increased by $1$ and it is checked whether the archive contains a solution feasible for the updated bound that is better than the current solution $\hat{x}$ of the algorithm. If so, the current solution $\hat{x}$ is updated with the best (now) feasible solution that can be found in the archive. The algorithm then proceeds with the next epoch consisting of $\tepoch$ steps for the updated bound and does so until $\tepoch$ steps have finally been carried out for the bound $B$ of the given problem. 
For our theoretical investigations, $\tepoch$ is the crucial parameter for the success probability and we assume that $B$ is a positive integer. Hence, we will mainly concentrate on $\tepoch$ as parameter and the total number of iterations $\temax$ can be obtained by considering that in total $B$ epochs of length $\tepoch$ are carried out. For our experiments the algorithm works with $\temax$ as an input and divides it (roughly) equally among the epochs. 
During the run, the algorithm stores at $x_{\hat B}$ the best feasible solution that it obtains for budget $\Bhat$, $0 \leq \Bhat \leq B$, and finally returns $x_{B}$ as the solution to the given problem with budget constraint $B$.

\begin{algorithm}[t]
\begin{algorithmic}[1]
\STATE Set $j:=0, x_j:=0^n$, $\Bhat:=0$, $A:=\emptyset$\;
\STATE $\tepoch = \lfloor (\temax/(\lceil B \rceil))\rfloor$
\WHILE{$(\Bhat \leq B) \wedge (t < \temax)$}
  \FOR{($k:=1$, $(k \leq \tepoch) \wedge (t < \temax)$, $k:=k+1$)}
    \STATE $y:=$ mutation($x_j$)
    \STATE $t:=t+1$
    \IF{$(c(y) > \Bhat) \wedge (c(y) \leq B)$}
    \IF{$\not \exists z \in A: (c(z) \leq c(y) \wedge f(z) > f(y))$ }
    \STATE $A:=A\cup\{y\}$
    \ENDIF
    \ENDIF
    \IF{$(c(y) \leq \Bhat) \wedge (f(y) \geq f(x_j))$}
      \STATE $x_j := y$
    \ENDIF
  \ENDFOR
  \STATE $A:=A\setminus\{y\in A: c(y) \leq \Bhat\}$
\STATE $\Bhat:= \min\{\Bhat + 1,B\}$
  \STATE $\hat{x} := x_j$
  \STATE $A^* := \{ y\in A: c(y) \leq  \Bhat\}$ 
  \IF{$|A^*| > 0$}
    \STATE $y^* := \arg\max\{ f(y): y\in A^*\}$ 
    \IF{$f(y^*) \geq f(\hat{x})$}
      \STATE $\hat{x}:=y^*$
    \ENDIF
  \ENDIF
  \STATE $j:=j+1$
  \STATE $x_{j}:=\hat{x}$
  
\ENDWHILE
\RETURN $x_j$
\end{algorithmic}
\caption{\oneone with archive, input: $f, c, B, \temax$}
\label{alg:1plus1-new}
\end{algorithm}

\subsection{Analysis}
Let $x^*_{\hat{B}}$ be an optimal solution for a reduced budget $R(B,c)$ dependent on the characteristics of the monotone cost function $c$ and given budget $\hat{B}$, $0 \leq \hat{B} \leq B$. For details on the budget reduction, we refer the reader to Equation 4 in~\cite{DBLP:conf/ijcai/QianSYT17}.
As done in \cite{NeumannWittECAI23}, we assume that $c \colon \{0,1\}^n \rightarrow \mathds{N}$ takes on non-negative integer values for our analysis.
We show that for each $\hat{B} \in \{0, \ldots, B\}$, \oneone with archive computes a solution $x_{\hat{B}}$ with 

$$c(x_{\hat{B}}) \leq \hat{B} \text{ and } f(x_{\hat{B}}) \geq \frac{\alpha_f}{2}(1- e^{-\alpha_f}) \cdot f(x^*_{\hat B})$$

where $x^*_{\hat B}$ is an optimal solution for budget $R(B,c)$.
Note that this matches the results given in Theorem 5 in \cite{DBLP:journals/ai/RoostapourNNF22} where it is shown that \gsemo computes for any possible budget up to the given budget $B$ a solution of the stated quality.

We now show that a \oneone using an archive for the solutions that exceed the current bound $\Bhat$ (but have cost at most $B$) is able to obtain the same approximation ratio as the multi-objective approaches presented in  \cite{DBLP:conf/ijcai/QianSYT17,NeumannWittECAI23}. 

Let $\delta_c = \min_{V' \subseteq V} \min_{v \not \in V'} (c(V' \cup \{v\}) - c(V')\geq 1$ the minimal possible cost increase when adding one element to any set not containing the element.

We  use 
$\tepoch=en \ln(nB^2)$ which implies $\temax \leq (B+1) \cdot \tepoch=en(B+1) \ln(nB^2) = O(nB(2\ln B +\ln n))$.

We denote by $x^*_{\Bhat}$ an optimal solution for the reduced constraint bound $R(B,c)$ dependent on $\Bhat$ and characteristics of the constraint function $c$ as done in \cite{DBLP:conf/ijcai/QianSYT17,NeumannWittECAI23}.
\begin{theorem}
\label{thm:general}
 Let $\temax = B \cdot \tepoch$ where $\tepoch \geq  en \ln (nB^2)$.
Then \oneone with archive computes with probability $1-o(1)$ for each bound $\Bhat$, $0 \leq \Bhat \leq B$, a solution $x_{\Bhat}$
with 
$$
c(x_{\Bhat}) \leq \Bhat \text{ and } f(x_{\Bhat}) \geq (\alpha_f/2) \cdot (1- e^{-\alpha_f}) \cdot f(x^*_{\Bhat})
$$
 In particular it computes a $(\alpha_f/2) (1- e^{-\alpha_f})$-approximation with probability $1-o(1)$ for the given bound $B$ when setting $\temax = B \cdot \tepoch$ with  $\tepoch\geq en \ln (nB)$.
\end{theorem}

\begin{proof}
We start with some observations. For each bound $B'$, $0 \leq B' \leq B$ the algorithm maintains a solution $x$ with $c(x) \leq B'$  that has the highest function value among all solution of cost at most $B'$ obtained during the run.
This solution is either contained in $A \cup \{\hat{x}\}$ or stored at $x_{\Bhat}$ when the bound $\Bhat=B'$ is reached.

According to \cite{DBLP:conf/ijcai/QianSYT17}, for a given bound $\Bhat$ a solution with the stated approximation quality can be obtained by selection the best out of two solutions.
\paragraph{Case 1:}
The first solution consists of the single element $v^* \in V$ of highest function value among all solution with a single element $v$ for which $c(v) \leq \Bhat$ holds.
If such a $v_{\Bhat}^*$ with $c(v^*) \leq \Bhat$ is a $(\alpha_f/2) (1- 1/e^{\alpha_f})$-approximation for bound $\Bhat$, then it can be obtained by flipping the bit corresponding to $v_{\Bhat}^*$ and no other bit in the initial solution $0^n$. The probability for this to happen in a single mutation step applied to $0^n$ is at least $1/(en)$.
The solution $0^n$ is chosen $\tepoch$ times for mutation and the probability that a feasible solution for $\Bhat$ with function value at least $f(v_{\Bhat}^*)$ has not been obtained is upper bounded by
\begin{equation}
\label{eq:prob}
    (1-1/en)^{\tepoch} \leq (1-1/en)^{en \ln (nB^2)} \leq 1/(nB^2).
\end{equation}

Using the union bound the probability that for at least one value of $\Bhat \in \{1, \ldots, B\}$ such a solution has not been obtained is $1/(nB)$.

\paragraph{Case 2:}
If selecting element $v_{\Bhat}^*$ only does not yield a $(\alpha_f/2) (1- 1/e^{\alpha_f})$-approximation for bound $\Bhat$, then a solution with the desired approximation quality can be obtained by incrementally adding an element with the largest marginal gain to the solution $y_{\Bhat}$
with 
\begin{equation} c(y_{\Bhat}) \leq C_{\Bhat} \text{ and }
f(y_{\Bhat}) \geq \left(1- e^{-\alpha_f C_{\Bhat}/B} \right) \cdot f(x^*_{\Bhat}).
\label{eq:progress}
\end{equation}
Note that the search point $0^n$ meets the condition for $C_{\Bhat}=0$.

For a given solution $x$, let $N(x)$ be the set of all solutions that can be obtained from $x$ by flipping a single $0$-bit in $x$.  Then we call
$$y = \arg \max_{z \in N(x)}(f(z) - f(x))/(c(z)-c(x))$$ 
a solution with the largest marginal gain with respect to $x$ and the considered objective and cost function. The element contained in $y$ but not in $x$ is an element with the largest marginal gain.

According to \cite{DBLP:conf/ijcai/QianSYT17}, adding an element with the largest marginal gain to $y_{\Bhat}$ results in a solution $y'_{\Bhat}$ with
\begin{eqnarray*}
f(y'_{\Bhat}) & \geq &  f(y_{\Bhat}) + \alpha_f \cdot \frac{c(y'_{\Bhat})-c(y_{\Bhat})}{\hat{B}} \cdot (f(x^*_{\Bhat}) - f(y_{\Bhat}))\\
& \geq &  f(y_{\Bhat}) + \alpha_f \cdot \frac{\delta_{\hat{c}}}{\hat{B}} \cdot (f(x^*_{\Bhat}) - f(y_{\Bhat}))\\
& \geq &  \left(1- \alpha_f \cdot \frac{\delta_{\hat{c}}}{\hat{B}}\right) f(y_{\Bhat}) + \alpha_f \cdot \frac{\delta_{\hat{c}}}{\hat{B}} \cdot f(x^*_{\Bhat})\\
& \geq &   \left(1- \alpha_f \cdot \frac{\delta_{\hat{c}}}{\hat{B}}\right) \left(1- e^{-\alpha_f C_{\Bhat}/\hat{B}} \right) \cdot f(x^*_{\Bhat}) + \alpha_f \cdot \frac{\delta_{\hat{c}}}{\hat{B}} \cdot f(x^*_{\Bhat})\\
& \geq &  \left(1- \left(1- \alpha_f \cdot \frac{\delta_{\hat{c}}}{\hat{B}}\right) \cdot e^{-\alpha_f C_{\Bhat}/B} \right)f(x^*_{\Bhat})\\
& \geq & \left(1-  e^{-\alpha_f \cdot \frac{\delta_{\hat{c}}}{\hat{B}}} \cdot e^{-\alpha_f C_{\Bhat}/\hat{B}}\right)f(x^*_{\Bhat}) \\
& \geq & \left(1- e^{-\alpha_f( C_{\Bhat}+ \delta_{\hat{c}})/\hat{B}} \right) \cdot f(x^*_{\Bhat})
\end{eqnarray*}

Note that for $C_{\Bhat}+ \delta_{\hat{c}} \geq \Bhat$, we have 
  $$ f(y'_{\Bhat}) \geq \left(1- e^{-\alpha_f( C_{\Bhat}+ \delta_{\hat{c}})/\hat{B}} \right) \cdot f(\hat{x}_{\Bhat})
\geq  (1- e^{-\alpha_f}) \cdot f(x^*_{\Bhat}).
  $$
. 

We consider the case when $c(y'_{\Bhat}) >\Bhat$. 
We have $f(v_{\Bhat}^*)\geq \alpha_f \cdot (f(y'_{\Bhat}) -f(y_{\Bhat}))$ as $f$ is $\alpha_f$-submodular which implies

$$f(v_{\Bhat}^*) + f(y_{\Bhat}) \geq \alpha_f \cdot f(y'_{\Bhat}) \geq \alpha_f \cdot
 \left(1- e^{-\alpha_f} \right) \cdot f(x^*_{\Bhat}) 
$$

and therefore $\max\{f(v_{\Bhat}^*),f(y_{\Bhat}))\} \geq  (\alpha_f/2) \cdot \left(1- e^{-\alpha_f} \right) \cdot f(x^*_{\Bhat}). $

If $c(y'_{\Bhat})>\Bhat$ and  $v_{\Bhat}^*$ is not a $(\alpha_f/2) (1- e^{-\alpha_f})$-approximation for bound $\Bhat$ then we have 
$$f(y_{\Bhat}) \geq (\alpha_f/2) (1- e^{-\alpha_f}) \cdot f(x^*_{\Bhat})
$$

For $\Bhat=0$, the solution $x_0=0^n$ has the desired approximation quality.
Let $B' \in \{1, \ldots, \Bhat\}$ and consider the solution $y_{B'} \in A \cup \{x_j\}$ that meets the condition for the largest possible value of $C_{B'}$ according to Equation~\ref{eq:progress}.

We claim that the invariant $C_{B'} \geq \Bhat$ holds during the run of the algorithm with high probability for any $B'$ and $\Bhat$ as long as the desired approximation has not been reached.
For $\Bhat=0$, the solution $x_0=0^n$ obviously fullfills the condition for $C_{B'}=0$ for any $B'$.

We consider the time when $\Bhat=C_{B'}$ holds for the first time and show that $C_{B'}$ increases with high probability. Note that this is a pessimistic assumption as $C_{B'}$ might increase earlier by creating an offspring of the current solution $x_j$.
We consider the epoch of the next $\tepoch$ steps once reached $\Bhat=C_{B'}$.
Note that in this epoch only solutions $x_j$ with $c(x_j) \leq C_{B'}$ and $f(x_j) \geq f(y_{\Bhat})$ are chosen for mutation which implies that they meet the condition of Equation~\ref{eq:progress} for $C_{B'}$.

The probability that there is no mutation step adding the element with the largest marginal gain to the current solution $x_j$ and therefore increasing $C_{B'}$ by producing an offspring $y$ (line 5 of Algorithm~\ref{alg:1plus1-new}) by at least $\delta_c \geq 1$  is at most $1/(nB^2)$ (see Equation~\ref{eq:prob}). 

The value of $C_{B'}$ needs to be increase at most $B$ times and hence we  obtain a solution $y_{B'}$ with  
$f(y_{B'}) \geq (\alpha_f/2) (1- e^{-\alpha_f}) \cdot f(x^*_{B'})
$
with probability at least $1-1/(nB)$ if element $v_{B'}^*$ from Case 1 does not give the desired approximation.

There are $B$ different values of $B'$ which implies that for all values of $\Bhat$, $0 \leq \Bhat \leq B$, 
a solution $x_{\Bhat}$ with
$f(x_{\Bhat}) \geq (\alpha_f/2) (1- e^{-\alpha_f}) \cdot f(x^*_{\Bhat})$ 
is obtained with probability at least $(1-1/n - 1/(nB) \geq 1-2/n$ within the run of the algorithm.
If the case where we are only interested in the desired approximation for the given bound $B$, choosing $\temax = B \cdot \tepoch$ with  $\tepoch\geq en \ln (nB)$ suffices as the element $v^*_B$ is introduced with probability at least $1/(nB)$ in this case and the solution $y_B$ is obtained with probability at least $1/n$ by considering at most $B$ increases of $C_B$.
\end{proof}

The previous result can be adapted to the case of a uniform constraint $|x|_1 \leq r$ as considered in Section~\ref{sec:noarcuniform}. As the cost of each item is $1$, we can obtain the following corollary by adjusting the proofs of Theorem~\ref{thm:uniform} and \ref{thm:general}.

\begin{corollary}
    Let $\temax = r \cdot \tepoch$ where $\tepoch \geq  2en \log n$.
Then \oneone with archive computes a $(1- 1/e)$-approximation with probability $1-o(1)$.
\end{corollary}

\begin{table*}[htbp]
 \caption{Maximum coverage scores obtained by $(1+1)$ EA with archive and $(1+\lambda)$-EA on medium instances}
\small
    \centering
    \begin{tabular}{|c|c|c||c|c|c|c|c||c|c|c|c|c|} \hline 
       & & & \multicolumn{5}{|c||}{\bfseries Uniform} & \multicolumn{5}{|c|}{\bfseries Random}\\ \hline
 & & &      \multicolumn{2}{|c}{\bfseries $(1+1)$ EA } & \multicolumn{2}{|c|}{\bfseries $(1+\lambda)$-EA} & & \multicolumn{2}{|c|}{\bfseries $(1+1)$ EA} & \multicolumn{2}{c|}{\bfseries $(1+\lambda)$-EA} & \\ 
 Graph & $B$ & $t_{\max}$ &     Mean & Std & Mean & Std & $p$-value & Mean & Std & Mean & Std & $p$-value  \\
\hline
 \multirow{12}{*}{ca-CSphd} &  10 &  100000 &  222 & 0.000 & 222 & 0.000 & 1.000 & 234 & 7.764 & 235 & 9.296 & 0.695\\
 &  10 &  500000 &  222 & 0.000 & 222 & 0.000 & 1.000 & 242 & 11.846 & 240 & 11.939 & 0.399\\
 &  10 &  1000000 &  222 & 0.000 & 222 & 0.000 & 1.000 & 245 & 12.735 & 240 & 12.294 & 0.165\\
 &  43 &  100000 &  600 & 0.728 & 599 & 0.711 & 0.174 & 601 & 12.206 & 605 & 10.715 & 0.206\\
 &  43 &  500000 &  600 & 0.000 & 600 & 0.000 & 1.000 & 614 & 12.788 & 608 & 11.688 & 0.038\\
 &  43 &  1000000 &  600 & 0.000 & 600 & 0.000 & 1.000 & 619 & 12.417 & 609 & 13.414 & 0.008\\
 &  94 &  100000 &  927 & 0.964 & 928 & 0.750 & 0.191 & 932 & 11.212 & 931 & 12.182 & 0.589\\
 &  94 &  500000 &  928 & 0.365 & 928 & 0.000 & 0.824 & 947 & 11.671 & 940 & 11.732 & 0.035\\
 &  94 &  1000000 &  928 & 0.000 & 928 & 0.000 & 1.000 & 950 & 10.884 & 946 & 12.128 & 0.041\\
 &  188 &  100000 &  1278 & 1.413 & 1279 & 1.081 & 0.264 & 1304 & 12.423 & 1300 & 13.157 & 0.198\\
 &  188 &  500000 &  1279 & 0.702 & 1279 & 0.809 & 0.584 & 1321 & 12.863 & 1317 & 12.944 & 0.198\\
 &  188 &  1000000 &  1279 & 0.629 & 1279 & 0.629 & 1.000 & 1326 & 12.805 & 1322 & 13.281 & 0.212\\
\hline
 \multirow{12}{*}{ca-GrQc} &  12 &  100000 &  505 & 7.158 & 506 & 5.649 & 0.923 & 524 & 19.509 & 524 & 23.457 & 0.947\\
 &  12 &  500000 &  510 & 0.000 & 510 & 0.000 & 1.000 & 551 & 24.536 & 544 & 19.123 & 0.193\\
 &  12 &  1000000 &  510 & 0.000 & 510 & 0.000 & 1.000 & 563 & 19.987 & 555 & 22.822 & 0.132\\
 &  64 &  100000 &  1512 & 6.872 & 1512 & 6.518 & 0.690 & 1554 & 22.583 & 1547 & 22.618 & 0.261\\
 &  64 &  500000 &  1526 & 4.312 & 1526 & 4.476 & 0.668 & 1608 & 27.234 & 1596 & 25.095 & 0.147\\
 &  64 &  1000000 &  1529 & 2.380 & 1529 & 1.954 & 0.807 & 1631 & 20.223 & 1609 & 20.411 & 0.000\\
 &  207 &  100000 &  2742 & 6.646 & 2738 & 9.464 & 0.076 & 2766 & 22.039 & 2759 & 18.237 & 0.209\\
 &  207 &  500000 &  2769 & 4.720 & 2765 & 5.845 & 0.008 & 2844 & 21.315 & 2834 & 15.824 & 0.072\\
 &  207 &  1000000 &  2773 & 5.045 & 2768 & 4.480 & 0.000 & 2864 & 19.825 & 2845 & 16.747 & 0.001\\
 &  415 &  100000 &  3565 & 6.066 & 3557 & 9.051 & 0.000 & 3581 & 14.673 & 3563 & 14.940 & 0.000\\
 &  415 &  500000 &  3606 & 4.066 & 3602 & 4.468 & 0.001 & 3650 & 11.331 & 3631 & 12.922 & 0.000\\
 &  415 &  1000000 &  3612 & 4.838 & 3605 & 4.219 & 0.000 & 3664 & 12.773 & 3646 & 13.563 & 0.000\\
\hline
 \multirow{12}{*}{Erdos992} &  12 &  100000 &  601 & 2.773 & 600 & 2.682 & 0.359 & 645 & 15.347 & 649 & 16.234 & 0.333\\
 &  12 &  500000 &  604 & 0.000 & 604 & 0.000 & 1.000 & 668 & 19.263 & 669 & 23.397 & 0.947\\
 &  12 &  1000000 &  604 & 0.000 & 604 & 0.000 & 1.000 & 687 & 23.483 & 678 & 21.527 & 0.183\\
 &  78 &  100000 &  2454 & 6.173 & 2453 & 6.653 & 0.318 & 2398 & 31.071 & 2451 & 30.601 & 0.000\\
 &  78 &  500000 &  2472 & 0.890 & 2472 & 0.968 & 0.953 & 2489 & 37.452 & 2507 & 32.866 & 0.042\\
 &  78 &  1000000 &  2472 & 0.819 & 2473 & 0.629 & 0.988 & 2522 & 38.457 & 2521 & 37.696 & 0.988\\
 &  305 &  100000 &  4707 & 8.677 & 4718 & 12.322 & 0.000 & 4563 & 27.503 & 4613 & 34.571 & 0.000\\
 &  305 &  500000 &  4771 & 2.096 & 4772 & 1.489 & 0.061 & 4723 & 25.279 & 4750 & 21.008 & 0.000\\
 &  305 &  1000000 &  4774 & 0.937 & 4775 & 0.845 & 0.034 & 4752 & 25.387 & 4767 & 22.048 & 0.018\\
 &  610 &  100000 &  5240 & 5.026 & 5234 & 6.182 & 0.000 & 5236 & 13.125 & 5203 & 18.913 & 0.000\\
 &  610 &  500000 &  5263 & 0.937 & 5262 & 1.196 & 0.001 & 5314 & 8.737 & 5299 & 10.183 & 0.000\\
 &  610 &  1000000 &  5264 & 0.730 & 5263 & 1.155 & 0.745 & 5324 & 8.568 & 5312 & 9.105 & 0.000\\
\hline

    \end{tabular}
    \label{tab:combinedS}
\end{table*}

\begin{table*}[htbp]
 \caption{Maximum coverage scores obtained by $(1+1)$ EA with archive and $(1+\lambda)$-EA on large instances}
\small
    \centering
    \begin{tabular}{|c|c|c||c|c|c|c|c||c|c|c|c|c|} \hline 
       & & & \multicolumn{5}{|c||}{\bfseries Uniform} & \multicolumn{5}{|c|}{\bfseries Random}\\ \hline
 & & &      \multicolumn{2}{|c}{\bfseries $(1+1)$ EA } & \multicolumn{2}{|c|}{\bfseries $(1+\lambda)$-EA} & & \multicolumn{2}{|c|}{\bfseries $(1+1)$ EA} & \multicolumn{2}{c|}{\bfseries $(1+\lambda)$-EA} & \\ 
 Graph & $B$ & $t_{\max}$ &     Mean & Std & Mean & Std & $p$-value & Mean & Std & Mean & Std & $p$-value  \\
\hline
 \multirow{12}{*}{ca-HepPh} &  13 &  100000 &  1807 & 26.271 & 1804 & 29.228 & 0.584 & 1845 & 51.305 & 1850 & 39.910 & 0.723\\
 &  13 &  500000 &  1842 & 15.291 & 1843 & 9.754 & 0.695 & 1913 & 47.774 & 1915 & 52.637 & 0.525\\
 &  13 &  1000000 &  1840 & 5.112 & 1842 & 8.632 & 0.657 & 1938 & 61.109 & 1919 & 50.702 & 0.126\\
 &  105 &  100000 &  4627 & 28.443 & 4611 & 33.286 & 0.050 & 4659 & 35.830 & 4671 & 43.630 & 0.196\\
 &  105 &  500000 &  4762 & 17.618 & 4757 & 19.511 & 0.315 & 4876 & 50.292 & 4860 & 41.453 & 0.176\\
 &  105 &  1000000 &  4787 & 11.162 & 4787 & 11.396 & 0.784 & 4938 & 39.003 & 4908 & 44.298 & 0.015\\
 &  560 &  100000 &  8514 & 25.487 & 8486 & 20.989 & 0.000 & 8550 & 44.136 & 8517 & 32.833 & 0.003\\
 &  560 &  500000 &  8766 & 10.316 & 8745 & 12.357 & 0.000 & 8887 & 29.469 & 8847 & 22.758 & 0.000\\
 &  560 &  1000000 &  8796 & 9.713 & 8783 & 11.683 & 0.000 & 8956 & 35.481 & 8906 & 29.424 & 0.000\\
 &  1120 &  100000 &  10200 & 20.861 & 10161 & 17.307 & 0.000 & 10205 & 24.322 & 10154 & 21.882 & 0.000\\
 &  1120 &  500000 &  10460 & 9.395 & 10438 & 10.915 & 0.000 & 10533 & 17.225 & 10488 & 19.822 & 0.000\\
 &  1120 &  1000000 &  10492 & 7.397 & 10470 & 8.358 & 0.000 & 10595 & 18.362 & 10550 & 15.359 & 0.000\\
\hline
 \multirow{12}{*}{ca-AstroPh} &  14 &  100000 &  2862 & 59.952 & 2865 & 48.219 & 0.988 & 2856 & 85.887 & 2892 & 92.832 & 0.130\\
 &  14 &  500000 &  2965 & 13.289 & 2972 & 9.649 & 0.019 & 3005 & 75.418 & 3025 & 70.585 & 0.274\\
 &  14 &  1000000 &  2978 & 4.249 & 2978 & 4.690 & 0.941 & 3038 & 80.257 & 3030 & 56.286 & 0.525\\
 &  133 &  100000 &  8337 & 60.951 & 8326 & 53.037 & 0.464 & 8391 & 58.454 & 8402 & 78.692 & 0.605\\
 &  133 &  500000 &  8679 & 28.093 & 8677 & 24.020 & 0.848 & 8805 & 54.412 & 8816 & 50.448 & 0.579\\
 &  133 &  1000000 &  8722 & 24.139 & 8725 & 16.979 & 0.779 & 8931 & 47.486 & 8912 & 53.551 & 0.072\\
 &  895 &  100000 &  15015 & 35.029 & 14948 & 37.983 & 0.000 & 15017 & 42.388 & 14973 & 40.257 & 0.000\\
 &  895 &  500000 &  15559 & 20.849 & 15538 & 16.778 & 0.000 & 15653 & 30.599 & 15609 & 35.076 & 0.000\\
 &  895 &  1000000 &  15638 & 10.944 & 15621 & 10.409 & 0.000 & 15779 & 33.990 & 15729 & 23.735 & 0.000\\
 &  1790 &  100000 &  17018 & 24.879 & 16936 & 23.255 & 0.000 & 17000 & 20.884 & 16921 & 26.926 & 0.000\\
 &  1790 &  500000 &  17438 & 8.670 & 17412 & 12.840 & 0.000 & 17488 & 15.484 & 17437 & 18.276 & 0.000\\
 &  1790 &  1000000 &  17491 & 8.581 & 17469 & 6.715 & 0.000 & 17579 & 11.757 & 17524 & 11.488 & 0.000\\
\hline
 \multirow{12}{*}{ca-CondMat} &  14 &  100000 &  1759 & 48.528 & 1753 & 68.135 & 0.807 & 1695 & 94.934 & 1758 & 95.933 & 0.010\\
 &  14 &  500000 &  1853 & 3.960 & 1853 & 4.925 & 0.569 & 1857 & 68.802 & 1885 & 62.486 & 0.048\\
 &  14 &  1000000 &  1856 & 3.108 & 1857 & 2.580 & 0.079 & 1875 & 65.595 & 1897 & 68.254 & 0.287\\
 &  146 &  100000 &  6633 & 56.195 & 6654 & 63.743 & 0.225 & 6546 & 81.784 & 6635 & 77.537 & 0.000\\
 &  146 &  500000 &  7030 & 21.828 & 7042 & 22.911 & 0.030 & 7050 & 54.530 & 7079 & 52.443 & 0.026\\
 &  146 &  1000000 &  7082 & 10.261 & 7079 & 14.222 & 0.610 & 7151 & 59.810 & 7176 & 70.114 & 0.179\\
 &  1068 &  100000 &  15744 & 58.965 & 15672 & 54.984 & 0.000 & 15721 & 61.110 & 15634 & 66.097 & 0.000\\
 &  1068 &  500000 &  16671 & 24.122 & 16655 & 28.693 & 0.013 & 16761 & 45.483 & 16728 & 44.291 & 0.006\\
 &  1068 &  1000000 &  16797 & 17.002 & 16789 & 17.744 & 0.149 & 16980 & 45.050 & 16924 & 45.781 & 0.000\\
 &  2136 &  100000 &  19150 & 32.710 & 18968 & 48.385 & 0.000 & 19130 & 36.406 & 18961 & 46.673 & 0.000\\
 &  2136 &  500000 &  20039 & 18.032 & 20003 & 21.860 & 0.000 & 20091 & 24.887 & 20025 & 25.165 & 0.000\\
 &  2136 &  1000000 &  20167 & 12.032 & 20148 & 10.722 & 0.000 & 20292 & 21.225 & 20219 & 22.946 & 0.000\\
\hline
 
    \end{tabular}   
    \label{tab:combinedL}
\end{table*}

\section{Experimental Investigations}\label{sec:Exp}

We now carry out experimental investigations for the two single-objective evolutionary approaches introduced for optimizing constrained submodular problems.
We evaluate the algorithms on instances of the NP-hard maximum coverage problem which is a classical submodular optimization problem defined on graphs.

We use standard-bit-mutation-plus~(see ~\cite{NeumannWittECAI23}) as mutation operator for the algorithm and run them on the same instances as done in~\cite{NeumannWittECAI23}. 
The underlying graphs are sparse graphs from the network repository. 
In the uniform case, all nodes of weight $1$ whereas in the random case each node have been assigned a cost independently of the others chosen uniformly at random in $[0.5, 1.5]$.
For each uniform and random setting we consider the same bounds $B$ and number of fitness evaluations $\temax$ as done in \cite{NeumannWittECAI23}, namely
$B= \log_2 n, \sqrt{n}, \lfloor n/20 \rfloor, \lfloor n/10 \rfloor$
and
$\temax=100000, 500000, 1000000$. 
A result is called
statistically significant if the $p$-value is at most $0.05$. 

We first consider the graphs ca-CSphd, ca-GrQc, and Erdos992 consisting of 1882, 4158, and
6100 nodes, respectively. The results are shown in Table~\ref{tab:combinedS}. For the graph ca-CSphd, the results obtained by the two algorithms are very similar in both the uniform and random setting. For the graph ca-GrQc, the results differ only for large values of $B$ and the statistical test also obtains small $p$-values in this case. Similar observations can be made for the graph Erdos992. Small $p$-values can usually be observed together with a better performance of \oneone with archive. This is especially the case for the random setting.

We now consider results for the larger graphs ca-HepPh, ca-AstroPh, ca-CondMat, which consist of $11204$, $17903$, $21363$ nodes, respectively. The results are shown in Table~\ref{tab:combinedL} and overall follow the same trend. Overall, the advantage of \oneone with archive over \onelambda can be observed again for the larger values of $B$. The difference in mean also becomes larger for both the uniform and random instances and statistical tests show a stronger difference for the large instances then for the medium size ones.

Overall, it can be observed that the quality of solutions obtained by \onelambda and \oneone with archive are very similar. Only for larger random instances with a large constraint bound, the \oneone with archive obtains significantly better results. Doing a cross comparison with the results presented in \cite{NeumannWittECAI23}, we observe the the results of \onelambda and \oneone with archive are better then the ones of \gsemo and slightly inferior then the ones of \swgsemo.
We can therefore say that our newly introduced simple single-objective algorithms provide a good alternative to the more complex multi-objective setups and even outperform standard approaches based on \gsemo for the considered instances.

\section{Conclusions}\label{sec:concl}
The maximization of submodular functions under constraints captures a wide range of NP-hard combinatorial optimization problems. In addition to classical greedy algorithms, Pareto optimzation approaches relaxing a given constraint into an additional objective have shown to obtain state of the art results from a theoretical and practical perspective. Contrary to this, it has been shown that standard single-objective approaches using the classical (1+1)~EA easily get stuck in local optima. 
We presented adaptive single-objective approaches increasing the set of feasible solution incrementally during the optimization process. For the \onelambda, we have shown that this leads to best possible theoretical performance guarantee for the case of a uniform constraint. For more general monotone cost constraints, we presented a \oneone with archive and have shown that this algorithm obtains state of the art theoretical guarantees. 
Our experimental investigations show that both algorithms perform quite similar and outperform the standard \gsemo approach for the considered settings.

\noindent\textbf{Acknowledgments.}
This work has been supported by the Australian Research Council (ARC) through grant
FT200100536.

\end{document}